\documentclass[journal,compsoc]{IEEEtran}

\usepackage[utf8]{inputenc} 
\usepackage[T1]{fontenc}    
\usepackage{hyperref}       
\usepackage{url}            
\usepackage{booktabs}       
\usepackage{amsfonts}       
\usepackage{nicefrac}       
\usepackage{microtype}      
\usepackage{bm}
\usepackage{amsmath,amssymb}
\usepackage{algorithm,algorithmic}
\usepackage{graphicx}
\usepackage{textcomp}
\usepackage{balance}
\usepackage{multirow}

\DeclareMathOperator*{\argmin}{argmin}
\DeclareMathOperator*{\argmax}{argmax}
\newcommand\norm[1]{\lVert #1 \rVert}
\newcommand\prox[1]{\operatorname{prox}_{#1}}
\newcommand\proj[1]{\operatorname{P}_{#1}}
\newcommand\minimize[1]{\underset{#1}{\operatorname{minimize}}}
\newcommand\st{\operatorname{subject\ to}}
\newcommand\sign{\operatorname{sign}}

\newcommand\pderivative[2]{\frac{\partial #1}{\partial #2}}
\newcommand\ind[1]{\mathbb{I}(#1)}

\newtheorem{theorem}{Theorem}
\newtheorem{lemma}{Lemma}
\newtheorem{proposition}{Proposition}
\newtheorem{corollary}{Corollary}
\newenvironment{proof}{\begin{IEEEproof}}{\end{IEEEproof}}

\newcommand\X{\bm{X}}
\newcommand\1{\bm{1}}
\newcommand\T{\mathsf{T}}
\newcommand\sJ{\mathcal{J}}
\newcommand\sM{\mathcal{M}}
\newcommand\sL{\mathcal{L}}
\newcommand\R{\mathbb{R}}

\begin{document}
%
\title{The fastest $\ell_{1,\infty}$ prox in the west}

\author{Benjam\'in~B\'ejar,~
        Ivan Dokmani\'c,~\IEEEmembership{Member,~IEEE,}
        and~Ren\'e Vidal,~\IEEEmembership{Fellow,~IEEE}
\thanks{B. B\'ejar and R. Vidal are with the Mathematical Institute for Data Science and Department of Biomedical Engineering of The Johns Hopkins University. E-mail: \{bbejar, rvidal\}@jhu.edu.}
\thanks{I. Dokmani\'c is with the Department of Electrical and Computer Engineering of
the University of Illinois at Urbana Champaign. E-mail: dokmanic@illinois.edu}}

\markboth{Preprint}%
{B\'ejar \MakeLowercase{\textit{et al.}}: The fastest $\ell_{1,\infty}$ prox in the west}

\IEEEtitleabstractindextext{%
\begin{abstract}
Proximal operators are of particular interest in optimization problems dealing with non-smooth objectives because in many practical cases they lead to optimization algorithms whose updates can be computed in closed form or very efficiently. A well-known example is the proximal operator of the vector $\ell_1$ norm, which is given by the soft-thresholding operator. 
In this paper we study the proximal operator of the mixed $\ell_{1,\infty}$ matrix norm and show that it can be computed in closed form by applying the well-known soft-thresholding operator to each column of the matrix. However, unlike the vector $\ell_1$ norm case where the threshold is constant, in the mixed $\ell_{1,\infty}$ norm case each column of the matrix might require a different threshold and all thresholds depend on the given matrix. We propose a general iterative algorithm for computing these thresholds, as well as two efficient implementations that further exploit easy to compute lower bounds for the mixed norm of the optimal solution. Experiments on large-scale synthetic and real data indicate that the proposed methods can be orders of magnitude faster than state-of-the-art methods.

\end{abstract}

\begin{IEEEkeywords}
Proximal operator, mixed norm, block sparsity.
\end{IEEEkeywords}}

\maketitle


\section{Introduction}\label{sec:introduction}

\IEEEPARstart{R}{ecent} advances in machine learning and convex optimization techniques  have led to very efficient algorithms for solving a family of regularized estimation problems. Sparsity, as a strong regularization prior, plays
a central role in many inverse problems and the use of sparsity-promoting norms as regularizers
has become widespread over many different disciplines of science and engineering.
One added difficulty is the non-differentiability of such priors, which prevents the use of
classical optimization methods such as gradient descent or Gauss-Newton methods \cite{nocedal-wright06,boyd-vandenberghe04}. Proximal algorithms
present an efficient alternative to cope with non-smoothness of the objective function. Furthermore, in many practical situations, simple closed-form updates of the variables of interest are possible. For an excellent review about proximal operators and algorithms see \cite{combettes-pesquet11} and the monographs \cite{bach-etal12ftml,parikh-boyd14fto}.

\subsection{Motivation}
Let $\bm{X} = [\bm{x}_1,\ldots,\bm{x}_m]\in\mathbb{R}^{n\times m}$ be a real matrix with columns $\bm{x}_i\in\mathbb{R}^n$. The mixed $\ell_{p,q}$ norm of $\bm{X}$ is defined over its columns as
\begin{equation}
	\norm{\bm{X}}_{p,q} = \Big(\sum_{i=1}^m\norm{\bm{x}_i}_p^q\Big)^{1/q}.
\end{equation}
Mixed norms such as the $\ell_{p,1}$ matrix norm $(p\geq 2)$ have been used to promote block-sparse structure in the variables of interest, and the larger $p$ the stronger the correlation among the rows of $\bm{X}$ \cite{obozinski-etal10ss}. In particular, the $\ell_{\infty,1}$ norm has been shown to be useful for estimating a set of covariate regressors in problems such as multi-task learning \cite{Turlach:2005ex,quattoni-etal09icml,obozinski-etal10ss}, and representative (exemplar) selection \cite{Elhamifar:CVPR12}. A general formulation for these type of problems is to minimize some convex loss function subject to norm constraints:
\begin{equation}\label{prob:model_problem}\begin{array}{ll}
\minimize{\bm{X}}   	&J(\bm{X})\\
\st			&\norm{\bm{X}}_{\infty,1}\leq\tau,
\end{array}
\end{equation}
where $\tau>0$ controls the sparsity level and $J(\cdot)$ is some convex loss function. Note that keeping $\norm{\bm{X}}_{\infty,1}$ small encourages whole columns of $\bm{X}$ to be zero. In this contribution, we are interested in efficiently solving problems of the form of \eqref{prob:model_problem}. A simple method to solve problem (\ref{prob:model_problem}) is to use a projected (sub)gradient descent method that computes the $k$th iteration estimate $\bm{X}^{(k)}$ as
\begin{eqnarray}\label{eq:pgd}
\bm{Z}&\gets&\bm{X}^{(k-1)} -\eta_k\partial J\big(\bm{X}^{(k-1)}\big)\\
\bm{X}^{(k)}&\gets&\proj{\norm{\cdot}_{\infty,1}\leq\tau}\big(\bm{Z}\big),
\end{eqnarray}
where $\eta_k$ is the stepsize at the $k$th iteration of the algorithm, $\partial J(\bm{X})$ denotes a subgradient ({\em i.e.,} the gradient if differentiable) of $J$ at $\bm{X}$, and where $\proj{\norm{\cdot}_{\infty,1}\leq\tau}(\bm{Z})$ denotes the projection of $\bm{Z}$ onto the $\ell_{\infty,1}$ ball of radius $\tau$. The main computational challenge of such method is to project onto the $\ell_{\infty,1}$ mixed norm which can be computed by a proximal mapping of the dual norm---the mixed $\ell_{1,\infty}$ norm ({\em i.e.,} the induced $\ell_1$ norm of $\bm X$ seen as a linear operator). In this paper, we address the problem of projecting onto the mixed $\ell_{\infty,1}$ norm via the computation of the proximal operator of its dual norm. This allows us to solve the class of problems in \eqref{prob:model_problem} that involve structured/group sparsity, namely those involving constraints on (projections onto) the mixed $\ell_{\infty,1}$ norm. The proximal mapping of the mixed $\ell_{1,\infty}$ norm is also applicable to the computation of minimax sparse pseudoinverses to underdetermined systems of linear equations \cite{pseudo-1,dokmanic-gribonval19siam-jmaa}.

\subsection{Prior work}

Since the main computational challenge in solving problems of the form \eqref{prob:model_problem} is in the computation of the projection onto the $\ell_{\infty,1}$ ball of a certain radius, it is then of practical importance to devise computationally efficient algorithms for computing such projections. 
An efficient method for computing these projections is proposed in \cite{quattoni-etal09icml}. The algorithm is based on sorting the entries of
the $n$ by $m$ data matrix in order to find the values that satisfy the optimality conditions of the projection problem. The complexity of the method is dominated by the sorting operation and therefore has an average complexity of $O(mn\log(mn))$. An alternative strategy is to use root-search methods such as those in \cite{sra11mlkdd,chau-etal18icassp} in order to find the optimal solution. Here we take an alternative approach and look at the proximal operator of the mixed $\ell_{1,\infty}$ matrix norm. Since the mixed $\ell_{1,\infty}$ and $\ell_{\infty,1}$ norms are duals of each other, a simple relationship can be established between the proximal operator and the projection operator (see Section \ref{sec:prox_background}). However, by looking at the proximal operator a better insight and understanding of the problem can be gained and exploited to accelerate the algorithms.  Contrary to root-search methods our method is exact (up to machine precision), does not require any thresholds to determine convergence, and it is guaranteed to find the optimal solution in a finite number of iterations.

\subsection{Contributions}

In this paper we study the proximal operator of the mixed $\ell_{1,\infty}$ matrix norm and 
show that it can be computed using a generalization of the well-known soft-thresholding operator from the vector to the matrix case. 
The generalization involves applying the soft-thresholding operator to each column of the matrix using a possibly different threshold for each column. Interestingly, all thresholds are related to each other via a quantity that depends on the given matrix. This is in sharp contrast to the vector case, where the threshold is constant and is given by the regularization parameter. To compute the proximal operator efficiently, we propose a general iterative algorithm based on the optimality conditions of the proximal problem. Our method is further accelerated by the derivation of easy to compute lower bounds on the optimal value of the proximal problem that contribute to effectively reduce the search space.
A numerical comparison with the state of the art of two particular implementations of our general method reveals the improved computational efficiency of the proposed algorithms. We also illustrate the application of our results to biomarker discovery for the problem of cancer classification from gene expression data. The code used to generate the results presented in this paper is made publicly available by the authors.


\section{Norms, projections, and proximal operators}\label{sec:prox_background}
In this section we present some background material that highlights the relationship between proximal operators, norms, and orthogonal projection operators.

Consider a non-empty closed convex set $\mathcal{C}\subset\R^n$. The orthogonal projection of a point $\bm{x}\in\R^n$ onto $\mathcal{C}$ is given by
	\begin{equation}\label{eq:projection_constrained}
		\proj{\mathcal{C}}(\bm{x}) = \argmin_{\bm{y}\in\mathcal{C}}\, \frac{1}{2}\norm{\bm{x}-\bm{y}}_2^2,
	\end{equation}
	where we have included an irrelevant $1/2$ factor for convenience in the exposition. Alternatively, we can also express the projection of a point as an unconstrained optimization problem as
	\begin{equation}\label{eq:projection_unconstrained}
		\proj{\mathcal{C}}(\bm{x}) = \argmin_{\bm{y}}\, \ind{\bm{y}\in\mathcal{C}}+\frac{1}{2}\norm{\bm{x}-\bm{y}}_2^2,
	\end{equation}
where we have moved the constraint into the objective by making use of the indicator function of a non-empty subset $\mathcal X\subset \R^n$, which is given by
	\begin{equation}\label{eq:indicator_function}
		\ind{\bm{x}\in\mathcal{X}}=\left\{\begin{array}{rl}0, &\bm{x}\in\mathcal{X}\\+\infty, &\textrm{otherwise}\end{array}\right..
	\end{equation}
Keeping in mind the definition of the projection operator given in (\ref{eq:projection_unconstrained}) as an unconstrained optimization problem, we are now ready to introduce the definition of the proximal operator. Let $f(\bm{x}):\R^n\mapsto \R$ be a lower semicontinuous convex function. Then, for every $\bm{x}\in\R^n$ the proximal operator $\operatorname{prox}_f(x)$ is defined as
	\begin{equation}\label{eq:proximal_operator}
		\prox{f}(\bm{x}) = \argmin_{\bm{y}}\, f(\bm{y})+\frac{1}{2}\norm{\bm{x}-\bm{y}}_2^2.
	\end{equation}
It is then clear, that the proximal operator can be regarded as a generalization of the projection operator ({\em e.g.,} replace $f(\bm{y})$ by the indicator function of a set $\mathcal{C}$). Note that, at every point, the proximal operator is the {\em unique} solution of an unconstrained convex optimization problem. Uniqueness of the proximal operator can be easily argued from the fact that the quadratic term in (\ref{eq:proximal_operator}) makes the optimization cost strictly convex.

An important particular case that often appears in practice is that where the function $f$ is a norm. For example, problems of the form of \eqref{eq:proximal_operator} appear in many learning and signal processing problems, where the quadratic term can be seen as a data-fidelity term while the function $f$ can be thought of as imposing some prior on the solution ({\em e.g.,} sparsity). The special case where $f$ is a norm has also a close connection to projections via the Moreau decomposition theorem as we shall describe next. Let $f:\mathcal{X}\subseteq\mathbb{R}^n\mapsto\mathbb{R}$ be a lower semicontinuous convex function, then its Fenchel conjugate $f^*$ is defined as
\begin{equation}
f^*(\bm{y}) = \sup_{\bm{x}\in \mathcal{X}}\,\big\{\langle\bm{y},\bm{x}\rangle - f(\bm{x})\big\}.
\end{equation}
The Moreau decomposition theorem relates the proximal operators of a convex function and its Fenchel conjugate, as stated next.
	\begin{theorem}[\cite{moreau65bsmf}]
	Let $f$ be a lower-semicontinuous convex function and let $f^*$ denote its Fenchel (or convex) conjugate, then
	\begin{equation}\label{eq:moreau_decomposition}
		\prox{f}(\bm{x}) + \prox{f^*}(\bm{x}) = \bm{x}.
	\end{equation}		
	\end{theorem}

For the special case where $f(\bm{x}) = \norm{\bm{x}}$ is a norm, it is well known that its Fenchel conjugate $f^*$ is given by
	\begin{equation}
		f^*(\bm{x}) = \ind{\norm{\bm{x}}_\ast\leq 1} =  \left\{\begin{array}{rl}
											0, &\norm{\bm{x}}_\ast \leq 1\\
											+\infty, &\textrm{otherwise}
											\end{array}\right.,
	\end{equation}
where $\norm{\cdot}_*$ is the dual norm of $\norm{\cdot}$ ({\em i.e.,} $\norm{\bm{z}}_* = \sup_{\bm{x}}\, \{\langle\bm{z},\bm{x}\rangle\, :\, \norm{\bm{x}}\leq 1\}$). That is, the Fenchel conjugate of a norm is the indicator function of the unit-norm ball of its dual norm (see for instance \cite{boyd-vandenberghe04} for a proof). Since the proximal operator of the indicator function of a set equals the orthogonal projection onto the set, it follows from \eqref{eq:moreau_decomposition} that
	\begin{equation}\label{eq:proximal-projection}
		\prox{\lambda\norm{\cdot}} = I - \proj{\norm{\cdot}_\ast\leq \lambda},
	\end{equation}
where $\proj{\norm{\cdot}_\ast\leq\lambda}$ denotes the projection onto the ball of radius $\lambda$ of the dual norm, and where $I$ is the identity operator.

Let $\bm{X}\in \mathbb{R}^{n\times m}$ then its mixed $\ell_{1,\infty}$ (induced $\ell_1$) norm is given by
\begin{equation}\label{eq:induced_norm}
	\lVert\bm{X}\rVert_{1,\infty} = \max_{\lVert\bm{u}\rVert_1 = 1} \lVert \bm{Xu}\rVert_1 = \max_i\, \lVert\bm{x}_i\rVert_1,
\end{equation}
where $\bm{x}_i$ corresponds to the $i$th column of matrix $\bm{X}$.
For the case of the induced $\ell_\infty$ operator norm we have the well-known relationship
\begin{equation}\label{eq:l1_linf}
	\lVert \bm{X} \rVert_\infty = \max_{\lVert\bm{u}\rVert_\infty = 1} \lVert \bm{Xu}\rVert_\infty = \lVert\bm{X}^\T\rVert_{1,\infty}.
\end{equation}
Also, recall the duality relationship between the $\ell_{\infty,1}$ norm and the mixed $\ell_{1,\infty}$ norm:
\begin{equation}\label{eq:l1inf}
	\norm{\bm X}_{\infty,1} = \sum_{i=1}^m \norm{\bm x_i}_\infty = \big(\norm{\bm X}_{1,\infty}\big)_\ast .
\end{equation}
Thus, without loss of generality, we will focus our analysis on the derivation of the proximal operator for the mixed $\ell_{1,\infty}$ norm, and derive expressions for the proximal operators of the induced $\ell_\infty$ and the projection operator onto the $\ell_{\infty,1}$ norm using the above relationships.


\section{Analysis of the mixed $\ell_{1,\infty}$ norm proximal operator}
\label{sec:prox_operator}
The relationship given in \eqref{eq:proximal-projection} makes it clear that finding the proximal operator of a norm amounts to knowing how to project onto the unit-norm ball of the dual norm and vice-versa.  In \cite{quattoni-etal09icml} the authors derived the optimality conditions for the projection onto the $\ell_{\infty,1}$ norm (see \eqref{eq:l1inf}) and proposed an algorithm for its computation based on sorting the entries of the matrix. Since these norms are duals of each other, the proximal operator for such norms can be readily computed based on \eqref{eq:proximal-projection}. In contrast, we look at the proximal operator itself and derive the optimality conditions. By doing so, we arrive at a more compact expression for the optimality conditions that generalizes the well-known soft-thresholding algorithm to the matrix case. Our analysis allows for a more intuitive interpretation of the proximal operator as well as the derivation of novel algorithms for its computation.

Given a matrix $\bm{V}\in\mathbb{R}^{n\times m}$, the proximal operator of the mixed $\ell_{1,\infty}$ norm with parameter $\lambda > 0$ is the solution to the following convex optimization problem:
\begin{equation}\label{eq:prox_operator}
	\prox{\lambda\norm{\cdot}_{1,\infty}}(\bm{V}) = \underset{\bm{X}}{\argmin}\, \displaystyle\norm{\bm{X}}_{1,\infty} + \frac{1}{2\lambda}\norm{\bm{X}-\bm{V}}_F^2.
\end{equation}
Using the definition of the mixed norm in \eqref{eq:induced_norm}, we can rewrite problem \eqref{eq:prox_operator} as the following
constrained optimization problem:
\begin{align}
\label{prob:prox}
\begin{split}
	\minimize{\bm{X},\,t} 	\quad &t + \frac{1}{2\lambda}\norm{\bm{X}-\bm{V}}_F^2\\
	\st				\quad &\norm{\bm{x}_i}_1 \leq t, \quad i=1,\dots, m.
\end{split}
\end{align}
By looking at the structure of problem \eqref{prob:prox} it is easy to derive the following result:
\begin{lemma}[Matched Sign]\label{lemma:matched_sign}
	The sign of the optimal solution $\bm{X}^\star$ of \eqref{prob:prox} must match the sign of $\bm{V}$, that is
	\begin{equation}
		\sign(\bm{X}^\star)=\sign(\bm{V})\,,
	\end{equation}
	where the $\sign(\cdot)$ function operates element-wise.
\end{lemma}
\begin{proof}
	The proof follows by contradiction. Assume $(\bm{X}^\star,t^\star)$ is the optimal solution to problem \eqref{prob:prox} and that there are some nonzero entries of $\bm{X}^\star$ that have the opposite sign to the corresponding entries in $\bm{V}$, \emph{i.e.,} $\sign(x_{ij}^\star) = -\sign(v_{ij})$ for some $ij$. Now, form the matrix $\bm{\tilde X} $ such that $\tilde x_{ij} = \sign(v_{ij}) |x_{ij}^\star|$. The point $(\bm{\tilde X},t^\star)$ is feasible and causes a reduction in the objective function since $\norm{\bm{\tilde X}-\bm{V}}_F < \norm{\bm{X}^\star-\bm{V}}_F$ while keeping the norm unchanged $\norm{\bm{X}^\star}_{1,\infty}=t^\star  =  \norm{\bm{\tilde X}}_{1,\infty}$. This contradicts the assumption that $\bm{X}^\star$ is the optimal solution. 
\end{proof}

Based on Lemma \ref{lemma:matched_sign} the problem of finding the proximal operator in \eqref{prob:prox} boils down to finding the magnitudes of the entries of the matrix $\bm{X}$. Therefore, we can formulate it as\footnote{Notice that this is a power allocation problem which belongs to the general family of waterfilling problems \cite{palomar-fonollosa05tsp}.}
\begin{align}
\label{prob:prox_nn}
\begin{split}
	\minimize{\bm{X},\,t} 	\quad & t + \frac{1}{2\lambda}\norm{\bm{X}-\bm{U}}_F^2\\
	\st				\quad &\1^\T\bm{x}_i \leq t, \quad \bm x_{i} \geq \bm 0, \quad i=1,\dots, m,
\end{split}
\end{align}
where $\bm{U} = [\bm{u}_1,\ldots,\bm{u}_m]\in\R^{n\times m}_+$ is a matrix with non-negative entries given by $u_{ij} = |v_{ij}|$. The following result determines the optimal solution of problem \eqref{prob:prox_nn} and, as a consequence, it also determines the proximal operator of the mixed $\ell_{1,\infty}$ norm:

\begin{proposition}\label{prop:optimality}
	The optimal solution $(\X^\star,t^\star)$ of problem \eqref{prob:prox_nn} is given by
	\begin{equation}\label{eq:Xstar}
		\X^\star = \Big[\bm{U}-\lambda\1\boldsymbol{\mu}^\T\Big]_+,
	\end{equation}
	and
	\begin{equation}\label{eq:tstar}
		t^\star = \frac{\sum_{i\in\sM^\star}\frac{1}{|\sJ_i^\star|}\sum_{j\in\sJ_i^\star} u_{ij} - \lambda}{\sum_{i\in\sM^\star}\frac{1}{|\sJ_i^\star|}},
	\end{equation}
	where $[\cdot]_+ = \max(\cdot,0)$, $\sM^\star = \{1\leq i\leq m\,:\,\1^\T\bm{u}_i\geq t^\star\}$ is the set of columns affected by thresholding, $\sJ_i^\star=\{1\leq j\leq n\,:\,u_{ij}-\lambda\mu_i^\star \geq 0\}$ is the set of indices of the non-zero entries of $\bm{x}_i^\star$, and 
	\begin{equation}\label{eq:mustar}
		\mu_i^\star = \Big[\frac{\sum_{j\in\sJ_i^\star} u_{ij} - t^\star}{\lambda|\sJ_i^\star|}\Big]_+,\ i=1,\dots,m.
	\end{equation}
	is the $i$th entry of the vector $\bm{\mu}\in\R^m$.
\end{proposition}
\begin{proof}
The Lagrangian of problem \eqref{prob:prox_nn} is given by
\begin{equation}\begin{split}
	\sL\big(\bm{X},t,\bm{\mu},\{\bm{\sigma}_i\}_{i=1}^m\big) = t &+ \frac{1}{2\lambda}\sum_{i=1}^m \|\bm{x}_i-\bm{u}_i\|^2\\ 
	&+\sum_{i=1}^m \mu_i(\1^\T\bm{x}_i - t) - \sum_{i=1}^m \bm{\sigma}_i^\T \bm{x}_{i}\ .
\end{split}\end{equation}
Since the problem is convex, the necessary and sufficient conditions for optimality are given by the KKT conditions:
\begin{itemize}

	\item Zero gradient of the Lagrangian
		\begin{align}
			\pderivative{\sL}{\bm{x}_k} &= \frac{1}{\lambda} (\bm{x}_k-\bm{u}_k) + \mu_k\1 - \bm{\sigma}_k = \bm{0},\ \forall k\label{eq:dLdx}\\
			\pderivative{\sL}{t} &= 1 - \sum_{i=1}^m\mu_i = 0\label{eq:dLdt}
		\end{align}
		
	\item Primal and dual feasibility
		\begin{align}
			\1^\T\bm{x}_k \leq t, \quad\bm{x}_{k} \geq \bm{0}, \quad &k=1,\dots, m\label{eq:pfeas}\\
			\bm{\mu} \geq \bm{0}, \quad \bm{\sigma}_k \geq \bm{0}, \quad &k=1,\dots, m\label{eq:dfeas}
		\end{align}
				
	\item Complementary slackness
		\begin{align}
			\mu_k(\1^\T\bm{x}_k - t)=0, \quad &k=1,\dots, m\label{eq:slackmu}\\
			\bm{\sigma}_k\odot \bm{x}_{k}=0, \quad &k=1,\dots, m,\label{eq:slacksigma}
		\end{align}
		where $\odot$ denotes element-wise product.
\end{itemize}
	We start by showing that equation \eqref{eq:Xstar} holds or equivalently, that every column $\bm{x}_k$ of $\bm{X}$ satisfies
	\begin{equation}\label{eq:xstar}
		\bm{x}_k = \big[\bm{u}_k - \lambda\mu_k\1\big]_+,\ k=1,\ldots,m.
	\end{equation}
	In order to do so, let $\sM = \{1\leq i\leq m\,:\,\1^\T\bm{u}_i\geq t\}$ be the set of columns that are affected by thresholding. Take for instance $\bm{x}_k$ for some $k\in\sM$, then we have
	\begin{equation}
		\bm{x}_k = \bm{u}_k - \lambda\mu_k\1 +\lambda\bm{\sigma}_k.
	\end{equation}
	In this case we can have $x_{kj} > 0$ which, by \eqref{eq:slacksigma}, \eqref{eq:pfeas}, \eqref{eq:dfeas} implies $\sigma_{kj} = 0$. Alternatively, we can have $x_{kj} = 0$ which means $u_{kj}-\lambda\mu_k < 0$. Therefore, both situations can be written in compact form as
	\begin{equation}\label{eq:xM}
		\bm{x}_k = \big[\bm{u}_k - \lambda\mu_k\1\big]_+,\ k\in\sM,
	\end{equation}
	where the thresholding operation $[\cdot]_+$ is applied element-wise. Alternatively, take $\bm{x}_k$ for some $k\notin\sM$ then from \eqref{eq:slackmu} it follows that $\mu_k = 0$ and hence, $\bm{x}_k = \bm{u}_k+\lambda\bm{\sigma}_k$. From \eqref{eq:slacksigma} and the fact that $\bm{u}_k \geq \bm 0$ it follows that $\bm{\sigma}_k = \bm{0}$ for all $k\notin\sM$. Therefore, we have that
	\begin{equation}\label{eq:xnotM}
		\bm{x}_k = \bm{u}_k,\ k\notin\sM.
	\end{equation}
	Since $\mu_k=0$ for $k\notin\sM$ we can put together \eqref{eq:xM} and \eqref{eq:xnotM} into a single expression as in \eqref{eq:xstar}.
	It remains now to derive an expression that relates $t$ and $\{\mu_k\}_{k=1}^m$. We know from \eqref{eq:slackmu} and the fact that $\mu_k\neq 0$ for $k\in\sM$ that
	\begin{equation}\label{eq:tmu}
		\1^\T\bm{x}_k = \sum_{j=1}^n \big[ u_{kj} - \lambda\mu_k \big]_+ = \sum_{j\in\sJ_k} (u_{kj} - \lambda\mu_k)  = t,\, k\in\sM,
	\end{equation}
	where we the set $\sJ_k$ denotes the non-zero entries of $\bm{x}_k$. Solving for $\mu_k$ in (\ref{eq:tmu}) leads to
	\begin{equation}
		\mu_k = \frac{\sum_{j\in\sJ_k} u_{kj} - t}{\lambda|\sJ_k|},\, k\in\sM.
	\end{equation}
	Recall that for $k\notin\sM$ we have $\mu_k=0$ and $\1^\T\bm{u}_k < t$ therefore, we can compactly express $\mu_k$ as
	$$ \mu_k = \Big[\frac{\sum_{j\in\sJ_k} u_{kj} - t^\star}{\lambda|\sJ_k|}\Big]_+,\ k=1,\dots,m,$$
	and we recover the expression in (\ref{eq:mustar}). Finally, using equation (\ref{eq:dLdt}) it is easy to check that
	$$t^\star = \frac{\sum_{k\in\sM}\frac{1}{|\sJ_k|}\sum_{j\in\sJ_k} u_{kj} - \lambda}{\sum_{k\in\sM}\frac{1}{|\sJ_k|}},$$
	which completes the proof.
\end{proof}
We are now ready to derive an expression for the proximal operator of the mixed $\ell_{1,\infty}$ norm as:
\begin{corollary}[Proximal Operator]\label{eq:prox_solution}
	The proximal operator in (\ref{eq:prox_operator}) is given by
	\begin{equation}\label{eq:prox_expression}
		\prox{\lambda\norm{\cdot}_{1,\infty}}(\bm{V}) = \sign(\bm{V})\odot \Big[|\bm{V}|-\lambda\1\boldsymbol{\mu}^\T\Big]_+,
	\end{equation}
	where $\boldsymbol{\mu}$ is given as in Proposition \ref{prop:optimality}.
\end{corollary}
\begin{proof}
	It follows directly from Lemma \ref{lemma:matched_sign} and Proposition \ref{prop:optimality}.
\end{proof}

The expression in Corollary \ref{eq:prox_solution} resembles very much the well-known soft-thresholding operator. In fact, the proximal operator of the mixed $\ell_{1,\infty}$ norm applies a soft-thresholding operation to every column of the matrix but with a different threshold value $\lambda\mu_i$ for each column $i=1,\ldots,m$ (see Fig. \ref{fig:original_vs_thresholded}). As expected, the above expression reduces to soft-thresholding for $m=1$:

\begin{figure*}
	\centering
		\includegraphics[width=\linewidth]{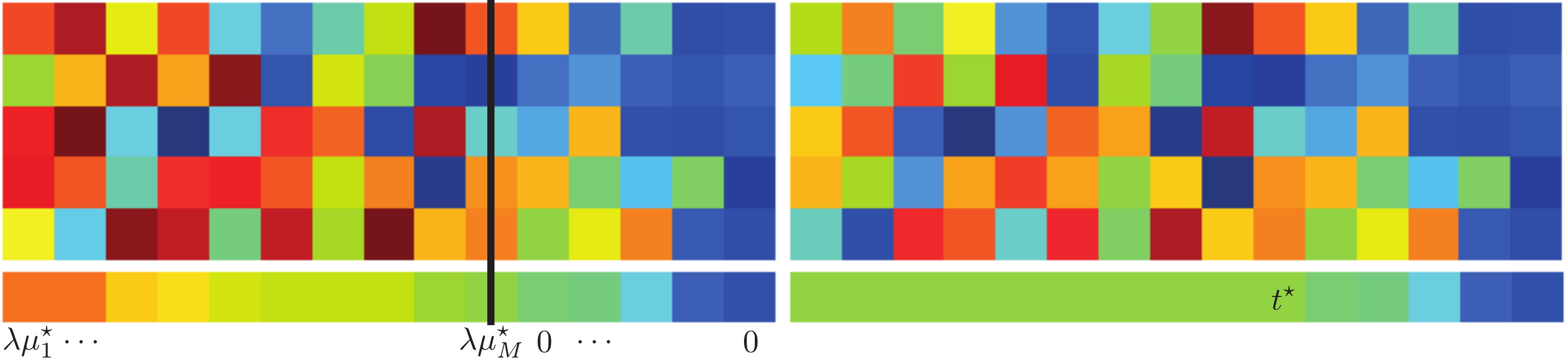}
	\caption{Illustration of the effect of the proximal operator. The left plot corresponds to the original matrix while the right plot corresponds to its thresholded counterpart. The bottom color plots represent the $\ell_1$ norm of each column. Warmer colors mean larger entries. The proximal operator projects the columns of the input matrix onto the $\ell_1$ ball of radius $t^\star$.}
	\label{fig:original_vs_thresholded}
\end{figure*}

\begin{corollary}[Soft-thresholding]
	In the case $m=1$ so that $\bm{V}=\bm{v}\in\mathbb{R}^n$ is a vector, the proximal operator is given by the well-known soft-thresholding
	\begin{equation}
		\prox{\lambda\norm{\cdot}_1}(\bm{v}) = \sign(\bm{v})\odot \Big[|\bm{v}|-\lambda\1\Big]_+.
	\end{equation}
\end{corollary}
\begin{proof}
By setting $m=1$ we get from Proposition \ref{prop:optimality} that $t = \sum_{j\in\sJ}|v_j| - |\sJ|\lambda$, where $\sJ$ is the set of non-zero entries of the optimal vector $\bm{x}^\star$. Substituting this value into (\ref{eq:mustar}) we get that $\mu^\star = 1$. The result then follows from Corollary~\ref{eq:prox_solution}.
\end{proof}

\begin{corollary}[Projection onto the $\ell_{\infty,1}$ ball]
	The projection onto the $\ell_{\infty,1}$ ball of radius $\lambda$ is
	\begin{equation}\label{eq:projection_l1oo}
		\proj{\norm{\cdot}_{\infty,1}\leq \lambda}(\bm{V}) = \sign(\bm{V})\odot \min\left(|\bm{V}|,\lambda\1\bm{\mu}^\T\right),
	\end{equation}
	with $\bm{\mu}$ given as in Proposition \ref{prop:optimality}.
\end{corollary}
\begin{proof}
The result follows from Corollary \ref{eq:prox_solution}, \eqref{eq:proximal-projection} and \eqref{eq:l1inf}.
\end{proof}

\paragraph*{Remark} Note that the results presented in this section can be trivially extended to the case of complex-valued matrices by interpreting the sign operation as extracting the phase of a complex number ({\em i.e.,} $\sign(\bm V) = \bm V/|\bm V|$).


\section{Algorithms for computing the mixed $\ell_{1,\infty}$ norm proximal operator\!}
\label{sec:algorithm}
The results in Proposition \ref{prop:optimality} and Corollary \ref{eq:prox_solution} give us the basis for finding an efficient algorithm for computing the mixed $\ell_{1,\infty}$ norm proximal operator. However, the computation of the proximal operator directly from those expressions requires knowledge about the optimal sets $\sM^\star$ and $\{\sJ_i^\star\}^m_{i=1}$, which are not known a priori. In this section we present a procedure for addressing this issue. But first, we describe an efficient pre-processing stage that can be used to reduce the search space for the optimal sets $\sM^\star$ and $\{\sJ_i^\star\}^m_{i=1}$ needed to compute the proximal operator in Proposition \ref{prop:optimality}. The idea is to maximize a lower bound on the mixed $\ell_{1,\infty}$ norm of the optimal solution, which allows us to discard columns that will not be affected by thresholding hence, reducing the search space for the optimal sets $\sM^\star$ and $\{\sJ_i^\star\}^m_{i=1}$. This allows us to effectively reduce the dimensionality of the problem since our algorithm will be then applied to a smaller matrix ({\em i.e.,} a matrix which contains a subset of columns of the original input matrix). After describing a procedure to maximize such lower bound we then propose a general iterative algorithm that uses the results in  Proposition \ref{prop:optimality} and Corollary \ref{eq:prox_solution} to find the right solution. 

\subsection{A lower bound on the norm}
It follows from the analysis presented in Section \ref{sec:prox_operator} that only a subset of the columns of the matrix $\bm{V}$ might be affected by the proximal operator ({\em i.e.,} those with $\ell_1$ norm larger than $t^\star$). This fact can be exploited to reduce the search space of the problem provided that some knowledge about the value of $t^\star$ is available. In particular, having a lower bound on $t^\star$ would allow us to discard columns with smaller $\ell_1$ norm. It turns out that a simple lower bound can be derived from the optimality conditions as stated in the following result:

\begin{lemma}[Lower-bound on the norm]\label{lemma:lb}
Let $\bm X^\star = \prox{\lambda\norm{\cdot}_{1,\infty}}(\bm{V})$ for some $\bm{V}\in\R^{n\times m}$. Let $t^\star = \norm{\bm{X}^\star}_{1,\infty}$ be the mixed $\ell_{1,\infty}$ norm of the optimal solution. Then, for any subset $\sM\subseteq\{1,\ldots,m\}$
	\begin{equation}\label{eq:t_lb}
		t_{\sM} = \frac{1}{|\mathcal{M}|}\big(\sum_{i\in\sM}\norm{\bm{v}_{i}}_1 - n\lambda\big)\leq t^\star .
	\end{equation}
\end{lemma}
\begin{proof}
From the optimality conditions of Problem \eqref{prob:prox_nn} we know that $\bm{x}^\star_i = [\bm{u}_i - \lambda\mu_i^\star\1]_+$, with $\bm{u}_i = |\bm{v}_i|$, and that $\1^\T\bm{x}^\star_i\leq t^\star$. Then, it follows that
\begin{equation}\begin{split}
t^\star 	&\geq \frac{1}{|\sM|}\sum_{i\in\sM}\1^\T\bm{x}^\star_i = \frac{1}{|\sM|}\sum_{i\in\sM}\1^\T[\bm{u}_i - \lambda\mu_i^\star\1]_+\\
			&\geq\frac{1}{|\sM|}\sum_{i\in\sM}\1^\T(\bm{u}_i - \lambda\mu_i^\star\1) = \frac{1}{|\sM|}\sum_{i\in\sM}(\1^\T\bm{u}_i - n\lambda\mu_i^\star)\\
			&\geq\frac{1}{|\sM|}\big(\sum_{i\in\sM}\1^\T\bm{u}_i - n\lambda\big) = \frac{1}{|\sM|}\big(\sum_{i\in\sM}\norm{\bm{v}_i}_1 - n\lambda\big),
\end{split}\end{equation}
where the last inequality follows from the fact that $\sum_{i\in\sM}\mu^\star_i \leq 1$.
\end{proof}

In order to reduce the search space of the problem, we can maximize the upper bound $t_{\sM}$ in \eqref{eq:t_lb} with respect to $\sM$. Since the sum $\sum_{i\in\sM}\norm{\bm{v}_{i}}_1$ is maximized when we choose the columns of $\bm{V}$ with the largest norm, a simple method to compute the set $\sM$ that maximizes $t_{\sM}$ is to sort the columns of $\bm{V}$ according to their $\ell_1$ norm, evaluate the objective for the top $k$ columns, and choose the value of $k$ that maximizes $t_\sM$, as described in Algorithm \ref{alg:max_lb}. Specifically, we form the vector $\bm{w}$ that contains the $\ell_1$ norms of the columns of $\bm{V}$ in decreasing order. From $\bm{w}$ we compute the partial sums $s_k = \sum_{i=1}^k w_i$ for $k=1,\ldots,m$, and evaluate the value of the bound \eqref{eq:t_lb} as described in Algorithm \ref{alg:max_lb} to find its maximizer.

\begin{algorithm}
\caption{Maximizing the lower bound \eqref{eq:t_lb} on $t^\star$}
\label{alg:max_lb}
\begin{algorithmic}[1]
\STATE Input: $(\bm{V},\lambda)$

\STATE Initialization: $\bm{U}\gets |\bm{V}|$ and $\bm{v}\gets \1^\T\bm{U}$

\STATE Sort by $\ell_1$ norm such that $w_1 \geq w_2 \geq\cdots\geq w_m$
$$\bm{w} \gets \operatorname{sort}(\bm{v})$$

\STATE Compute the maximizer
\begin{align*}
	s_k 	&\gets \sum_{i=1}^k w_i,\, k =1,\ldots,m\\
	t 			&\gets \max_{1\leq k\leq m}\big( (s_k-n\lambda)/k\big)
\end{align*}
\STATE Return: $t$
\end{algorithmic}\end{algorithm}

Note however, that while Algorithm \ref{alg:max_lb} allows us to efficiently find a maximum lower bound $t_\sM$ for $t^\star$, depending on the parameter $\lambda$, the maximum lower bound 
might be smaller than zero, in which case it is not useful. In such case, an alternative lower bound for $t^\star$ is given by the following result:
\begin{lemma}\label{lemma:clb}
Let $\bm X^\star = \prox{\lambda\norm{\cdot}_{1,\infty}}(\bm{V})$ for some $\bm{V}\in\R^{n\times m}$. Let $t^\star = \norm{\bm{X}^\star}_{1,\infty}$ be the mixed $\ell_{1,\infty}$ norm of the optimal solution. Then, it holds that
\begin{equation}\label{eq:t_lb2}
	\frac{1}{m}\big(\norm{\bm V}_{\infty,1}-\lambda\big) = \frac{1}{m}\big(\sum_{i=1}^m \norm{\bm v_i}_\infty - \lambda\big) \leq t^\star ,
\end{equation}
\end{lemma}
\begin{proof}
From \eqref{eq:dLdt} we know that $\sum_{i=1}^m \mu_i^\star = 1$. It also holds that $t^\star \geq \1^\T\bm x_i^\star$, hence
\begin{equation}\begin{split}
m\,t^\star  	&\geq \sum_{i=1}^m\1^\T\bm{x}^\star_i = \sum_{i=1}^m\1^\T[\bm{u}_i - \lambda\mu_i^\star\1]_+\\
			&\geq\sum_{i=1}^m \big(\max_{1\leq j\leq n} u_{ij} - \lambda\mu_i^\star\big) =
			 \sum_{i=1}^m \norm{\bm u_{i}}_\infty - \lambda\sum_{i=1}^m\mu_i^\star\\
			&=\norm{\bm U}_{\infty,1} - \lambda = \norm{\bm V}_{\infty,1} - \lambda.
\end{split}\end{equation}
\end{proof}

Note that the bound in \eqref{eq:t_lb2} will be negative only if the optimal solution is the zero matrix since:
\begin{equation}\begin{split}
	\norm{\bm V}_{\infty,1} < \lambda\; &\Longrightarrow\; \proj{\norm{\cdot}_{\infty,1}\leq\lambda}(\bm V) = \bm V\\
	&\Longrightarrow\; \prox{\lambda\norm{\cdot}_{1,\infty}}(\bm V) = \bm 0.
\end{split}\end{equation}

\subsection{A general algorithm}
A general procedure for computing the proximal operator of the mixed $\ell_{1,\infty}$ norm can be devised based on the optimality conditions of Proposition \ref{prop:optimality} and the observation that, for a fixed $t$, the problem in \eqref{prob:prox_nn} boilds down to projecting the columns of $\bm U$ onto the $\ell_1$ ball of radius $t$. A possible strategy for finding $t^\star$ is to start with a lower bound $t$ for $t^\star$, project each column of $U$ whose $\ell_1$ norm is above the current lower bound onto the $\ell_1$ ball of radius $t$, update the value of the lower bound using \eqref{eq:tstar}, and keep iterating until there are no further changes in $t$ (see Algorithm~\ref{alg:prox_projections}). This algorithm is guaranteed to converge to the optimal solution, as stated next.

\begin{algorithm}
\caption{Proximal operator of mixed $\ell_{1,\infty}$ norm: $\prox{\lambda\norm{\cdot}_1}(\bm V)$}
\label{alg:prox_projections}
\begin{algorithmic}[1]

\STATE Initialization: $\bm U\gets |\bm V|$
\STATE Compute lower bound on $t$
\STATE \textbf{do}
\STATE $\sM$-update: $\sM\gets\big\{i\,|\, t < \norm{\bm{u}_i}_1\big\}$
\FOR{$i\in\sM$}
\STATE Projection onto the simplex: $\bm{x}_i\gets \proj{\norm{\cdot}_1\leq t}(\bm u_i)$
\STATE $\sJ_i$-update: $\sJ_i\gets\big\{j\,|\, x_{ij} > 0 \big\}$
\ENDFOR
\STATE $t$-update: $t = \frac{\sum_{i\in\sM}\frac{1}{|\sJ_i|}\sum_{j\in\sJ_i} u_{ij} - \lambda}{\sum_{i\in\sM}\frac{1}{|\sJ_i|}}$
\STATE \textbf{while} $\sM$ or $\{\sJ_i\}_{i=1}^m$ change
\STATE Compute proximal operator using Corollary \ref{eq:prox_solution}.

\end{algorithmic}
\end{algorithm}

\begin{proposition}[Convergence]\label{prop:convergence}
Algorithm \ref{alg:prox_projections} converges to the proximal operator of the mixed $\ell_{1,\infty}$ norm of matrix $\bm{V}$ in a finite number of iterations.
\end{proposition}
\begin{proof}
Observe that the algorithm produces a monotonic sequence of values for $t$ that eventually converges to the optimal value $t^\star$.  To see this, note that for a given $t$, the projection onto the $\ell_1$ ball has the form of \eqref{eq:xstar} that is, $\bm x_i = \big[\bm u_i - \lambda\mu_i\1\big]_+$ for some value $\mu_i$. Let $\mu_i^\star$ denote the value at the optimal solution. Now since $t$ is a lower bound on $t^\star$ then it is necessary the case that $\mu_i \geq \mu_i^\star$ (hence $\sJ_i\subseteq \sJ_i^\star$). Let $\sJ_i$ denote the resulting sets after projecting onto the $\ell_1$ ball of radius $t$. The new value $t^+$ is then given by
\begin{equation}\begin{split}
	t^+ 	&= \frac{\sum_{i\in\sM}\frac{1}{|\sJ_i|}\sum_{j\in\sJ_i} u_{ij} - \lambda}{\sum_{i\in\sM}\frac{1}{|\sJ_i|}}\\
	&= \frac{\sum_{i\in\sM}\frac{1}{|\sJ_i|}\sum_{j\in\sJ_i} \big(u_{ij} -\mu_i\lambda +\mu_i\lambda\big) - \lambda}{\sum_{i\in\sM}\frac{1}{|\sJ_i|}}\\
			&= \frac{\sum_{i\in\sM}\frac{1}{|\sJ_i|} t + \mu_i\lambda - \lambda}{\sum_{i\in\sM}\frac{1}{|\sJ_i|}} = t + \lambda\frac{\sum_{i\in\sM}\mu_i-1}{\sum_{i\in\sM}\frac{1}{|\sJ_i|}}\geq t,
\end{split}\end{equation}
where the last inequality follows from the fact that $\mu_i\geq\mu_i^\star$ and $\sum \mu_i^\star = 1$ (see \eqref{eq:dLdt}). In fact, the inequality is strict and it is satisfied with equality only when $t=t^\star$ since in that case the optimality conditions of Proposition \ref{prop:optimality} are satisfied. Note that a change in $t$ can only happen if there is a change in the sets $\sJ_i$ or $\sM$. Since the number of possible sets is finite and due to the monotonicity of the values of $t$ the algorithm terminates in a finite number of iterations.
\end{proof}


\section{Numerical experiments}\label{sec:experiments}

\subsection{Complexity and implementation}
The complexity of Algorithm \ref{alg:prox_projections} depends on the method used for computing the projection step onto the simplex and there exist different alternatives in the literature \cite{condat16mp}. A naive implementation of the proposed algorithm can lead to a computationally inefficient method if at every iteration the projection step is computed from scratch. Alternatively, one could exploit previous estimates from one iteration to the next in order to improve the computational efficiency. In this paper we propose two different approaches for computing the projection step onto the simplex: the first one is based on sorting the columns of the matrix of absolute values that are affected by thresholding. In such case, the expected complexity of the method is dominated by the sorting operation and it is of $O(mn\log n)$ operations. The second approach is based on an active set method based on Proposition \ref{prop:optimality}. In fact, the projection onto the simplex part is equivalent to the one in \cite{michelot86ota}. While in the latter case the complexity analysis is not straightforward, we have experimentally observed it to be more efficient than the sorting-based one in most of the tested scenarios.

\subsection{Numerical validation}
In order to evaluate the computational complexity of the proposed algorithms we randomly generate matrices in $\mathbb{R}^{n\times m}$ with independent and identically distributed random entries drawn from a uniform distribution $\mathcal{U}([-0.5,0.5])$. We then apply the proposed implementations of Algorithm \ref{alg:prox_projections} and label them ``Sort'' for the sorting-based implementation and ``Active Set'' for the one based on active sets to compute projections onto the mixed $\ell_{\infty,1}$ ball for different radius values. We also compute the projections using the state of the art algorithms. In particular we compare to the method proposed in \cite{quattoni-etal09icml} which we denote as ``QT'' and with the recently proposed root-search based methods of \cite{chau-etal18icassp,chau-etal19siam-is} which we denote as ``ST'' (Steffensen) and ``NT'' (Newton), respectively. We record the execution time for different configurations (sizes) of the data matrix and for different values of the $\ell_{\infty,1}$ ball radius. In our experiments, we choose the radius of the ball to be a fraction $\alpha\in[0,1]$ of the true mixed norm of the matrix and compute the average computation time over $100$ realizations. For the methods in \cite{quattoni-etal09icml,chau-etal18icassp,chau-etal19siam-is} we use the implementations provided by the authors. The results for different matrix sizes are displayed in Table \ref{tab:timing}. As it can be observed from the table, our two implementations achieve the best performance offering an improvement over the state of the art that ranges between one and two orders of magnitude.

\begin{table*}[h]
\centering
\caption{Average execution time of the different methods in computing the projection onto the $\ell_{\infty,1}$ ball. The computation time corresponds to an average over $100$ realizations.}
\begin{tabular}{ccccccc}
\toprule
Size &$\alpha$	&ST	&NT &QT	&Sort	&Active Set\\
\toprule
\multirow{4}{*}{$100\times100$}
&$10^{-4}$ & 3.51E-03 & 3.46E-03 & 1.69E-03 & 1.03E-04 & \textbf{ 9.61E-05}\\
&$10^{-3}$ & 1.33E-02 & 9.17E-03 & 1.70E-03 & 1.44E-04 & \textbf{ 1.35E-04}\\
&$10^{-2}$ & 5.11E-02 & 3.04E-02 & 1.71E-03 & 3.74E-04 & \textbf{ 3.64E-04}\\
&$10^{-1}$ & 5.49E-02 & 3.65E-02 & 1.73E-03 & 7.02E-04 & \textbf{ 6.94E-04}\\
\midrule
\multirow{4}{*}{$1000\times100$}
&$10^{-4}$ & 9.23E-03 & 8.42E-03 & 1.97E-02 & 1.06E-03 & \textbf{ 1.04E-03}\\
&$10^{-3}$ & 3.81E-02 & 2.69E-02 & 1.96E-02 & 2.49E-03 & \textbf{ 2.47E-03}\\
&$10^{-2}$ & 1.19E-01 & 7.80E-02 & 1.96E-02 & \textbf{ 7.98E-03} & 8.01E-03\\
&$10^{-1}$ & 1.16E-01 & 8.21E-02 & 1.96E-02 & \textbf{ 9.25E-03} & 9.28E-03\\
\midrule
\multirow{4}{*}{$100\times1000$}
&$10^{-4}$ & 8.60E-02 & 6.25E-02 & 1.94E-02 & 8.88E-04 & \textbf{ 8.63E-04}\\
&$10^{-3}$ & 4.46E-01 & 2.60E-01 & 1.94E-02 & 1.28E-03 & \textbf{ 1.26E-03}\\
&$10^{-2}$ & 4.88E-01 & 3.19E-01 & 1.96E-02 & 3.56E-03 & \textbf{ 3.55E-03}\\
&$10^{-1}$ & 6.01E-01 & 3.96E-01 & 1.98E-02 & 6.82E-03 & \textbf{ 6.82E-03}\\
\midrule
\multirow{4}{*}{$1000\times1000$}
&$10^{-4}$ & 4.48E-01 & 3.28E-01 & 2.41E-01 & 1.11E-02 & \textbf{ 1.09E-02}\\
&$10^{-3}$ & 1.41E+00 & 9.23E-01 & 2.40E-01 & \textbf{ 2.61E-02} & 2.62E-02\\
&$10^{-2}$ & 1.60E+00 & 1.14E+00 & 2.40E-01 & 8.24E-02 & \textbf{ 8.22E-02}\\
&$10^{-1}$ & 1.61E+00 & 1.17E+00 & 2.40E-01 & \textbf{ 9.56E-02} & 9.56E-02\\
\midrule
\multirow{4}{*}{$10000\times1000$}
&$10^{-4}$ & 1.55E+01 & 1.01E+01 & 3.41E+00 & 1.18E-01 & \textbf{ 1.18E-01}\\
&$10^{-3}$ & 1.53E+01 & 1.06E+01 & 3.30E+00 & \textbf{ 2.62E-01} & 2.63E-01\\
&$10^{-2}$ & 1.83E+01 & 1.31E+01 & 3.33E+00 & \textbf{ 8.31E-01} & 8.31E-01\\
&$10^{-1}$ & 1.97E+01 & 1.42E+01 & 3.33E+00 & \textbf{ 9.64E-01} & 9.64E-01\\
\bottomrule
\label{tab:timing}
\end{tabular}
\end{table*}

\subsection{Application to cancer classification from gene expression data}
In this section we test our algorithms in the context of multi-task learning for the problem of cancer classification from gene expression data where the dimensionality of the feature vectors $m$ is typically much larger than the number of samples $p$. We use the datasets provided in \cite{nie-etal10neurips} which consist of five curated datasets of different types of cancers as described in \cite{nie-etal10neurips}. The datasets are briefly summarized in Table \ref{tab:datasets}. We pose the classification problem as a multi-task learning problem. In particular, given a dataset of points with associated labels $\mathcal{D}=\big\{(\bm x_i, c_i)\big\}_{i=1}^p$, with $\bm x_i \in \mathbb{R}^m$ and $c_i \in \{0,\ldots,n\}$, where $n$ is the number of classes, we build a data matrix $\bm X = [\bm x_1,\ldots, \bm x_p]^T$ and target label matrix $\bm Y = [\bm y_1,\ldots,\bm y_p]^T$ with
\begin{equation}
\bm y_i = [y_{i1},\ldots,y_{in}]^\T,\quad y_{ij} = \left\{\begin{array}{lc}1\; &j=c_i\\0\; &\textrm{else}\end{array}\right..
\end{equation}

\begin{table}[h]
\centering
\caption{Characterization summary of the used datasets, see \cite{nie-etal10neurips}.}
\begin{tabular}{lccc}
\toprule
Dataset &Classes $n$ &Samples $p$ &Dimension $m$\\
\midrule
Carcinom	\cite{su-etal01cr,yang-etal06bmcbi} &11 &174 &9182\\
GLIOMA \cite{nutt03cr}	&4 &50 &4434\\
LUNG \cite{bhattacharjee-etal01pnas}	&5 & 203 & 3312\\
ALLAML \cite{fodor97science}	&2 & 72 &7129\\
Prostate-GE \cite{singh-etal02cancercell} &2 & 102 & 5966\\
\bottomrule
\label{tab:datasets}
\end{tabular}
\end{table}

The problem is to predict the correct label for each class while enforcing feature sharing among them:
\begin{equation}\label{prob:feature_selection}
\begin{array}{ll}
\minimize{\bm W} &\lVert \bm Y - \bm X\bm W^T\rVert_F^2\\
\st	&\lVert \bm W \rVert_{\infty,1} \leq \tau
\end{array}.
\end{equation}
Note that problem \eqref{prob:feature_selection} falls within the family of problems in \eqref{prob:model_problem} which can be solved using a projected gradient descent strategy. For the projection step onto the $\ell_{\infty,1}$ ball we use the sorting-based implementation of Algorithm \ref{alg:prox_projections}.

We conducted an experiment using the datasets of Table \ref{tab:datasets} where we center the data points (mean subtraction) and normalize them by dividing each coordinate by its standard deviation. For each dataset we split the data into $80\%$ training and $20\%$ testing and computed the average classification performance over $100$ random data splits. Once we solve \eqref{prob:feature_selection} we use the following simple classification rule:

\begin{equation}\label{eq:classification_rule}
\hat c_i = \argmax_{1\leq j\leq n}\; \hat y_{ij},\quad \hat{\bm Y} = \bm X\bm W^T = [\hat{\bm y}_1,\ldots,\hat{\bm y}_p]^T.
\end{equation}

In addition, we use the learned weights to identify relevant features and train a (kernel) support vector machine (SVM) classifier on the identified features. Features are sorted according to the Euclidean norm of the columns of $\bm W$ being the most relevant index the one with larger norm. For the multi-class problem we use a {\em one-versus-one} strategy with majority voting. We also provide a comparison with the $\ell_{2,1}$ norm based feature selection method of \cite{nie-etal10neurips} for which we used the implementation provided by the authors. The $\ell_{\infty,1}$ ball radius $\tau$ in \eqref{prob:feature_selection} as well as the regularization parameter for the method in \cite{nie-etal10neurips} were chosen using a grid search. The average classification accuracy of both methods the classification rule \eqref{eq:classification_rule} are summarized in Table \ref{tab:accuracy}. As it can be appreciated we observe that the proposed method using the $\ell_{\infty,1}$ norm achieves better classification accuracy than the method based on the $\ell_{2,1}$ proposed in \cite{nie-etal10neurips}. It is important to note that the differences are more pronounced in multi-class problems than in binary ones indicating as expected, that the $\ell_{\infty,1}$ norm encourages the discovery of variables that are most correlated.

\begin{table*}[h]
\centering
\caption{Average classification accuracy using criterion \eqref{eq:classification_rule}.}
\begin{tabular}{lccccc}
\toprule
Dataset &Carcinom &GLIOMA &LUNG &ALLAML &Prostate-GE\\
					&[11 classes] &[4 classes] &[5 classes] &[2 classes] &[2 classes]\\
\midrule
$\ell_{2,1}$ (Nie et al.)	&95.50	&68.90 &76.95 &92.36 &93.25\\
$\ell_{\infty,1}$ (Proposed) &\textbf{97.74} &\textbf{78.50} &\textbf{83.28} &\textbf{95.07} &\textbf{93.65}\\
\bottomrule
\label{tab:accuracy}
\end{tabular}
\end{table*}

We also report the classification results using an SVM classifier and for different number of features used. The results are displayed in Fig. \ref{fig:svm_results} for all datasets. We can observe the superior performance of the proposed scheme in selecting relevant features for the discrimination task. Again the performance gap is generally more pronounced on those datasets with more than two classes. 

\begin{figure*}
\includegraphics[width=\linewidth]{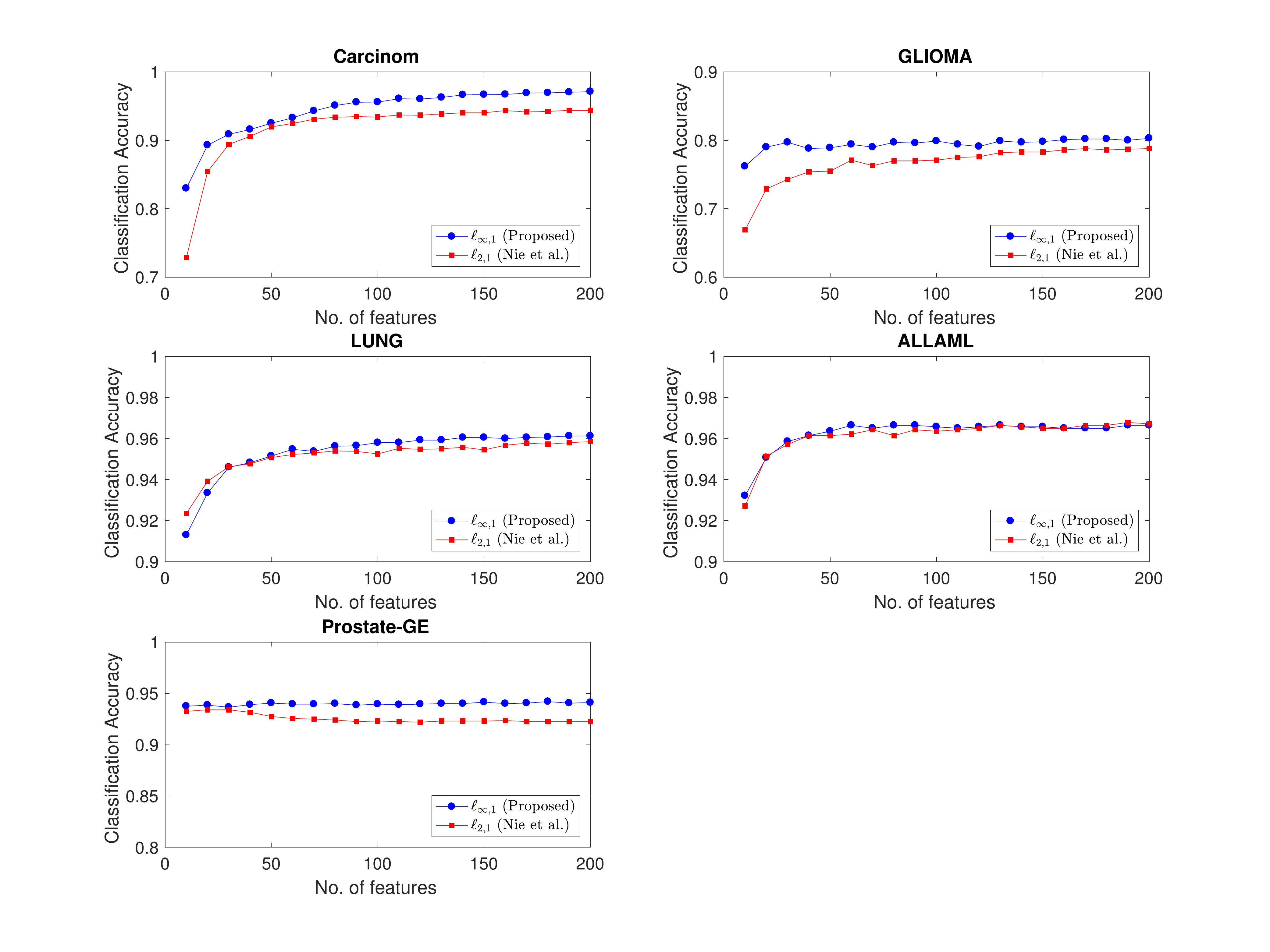}
\caption{Average classification results using an SVM classifier as a function of the number of features selected.}
\label{fig:svm_results}
\end{figure*}


\section{Conclusions}\label{sec:conclusions}
In this paper we have analyzed in detail the proximity operator of the mixed $\ell_{1,\infty}$ matrix norm. We have provided simple expressions for its computation that generalize the well-known soft-thresholding algorithm. By exploiting the duality relationship to the $\ell_{\infty,1}$ norm we also derive the projection operator onto the mixed $\ell_{\infty,1}$ norm. In addition, we have proposed a general algorithm for the computation of the proximal operator and two particular implementations that can be orders of magnitude faster than the state of the art making them particularly suitable for large-scale problems. We have also illustrated the application of the $\ell_{\infty,1}$ norm for biomarker discovery (feature selection) for the problem of cancer classification from gene expression data.

\IEEEpeerreviewmaketitle

\section*{Acknowledgment}
This research was supported in part by the Northrop Grumman
Mission Systems Research in Applications for Learning Machines
(REALM) initiative.

\ifCLASSOPTIONcaptionsoff
  \newpage
\fi

\balance



\end{document}